\documentclass[twoside]{article}

\usepackage[accepted]{arxiv}
\usepackage{times}
\usepackage{graphicx} 
\usepackage{natbib}
\usepackage{algorithm}
\usepackage{algorithmic}
\usepackage{hyperref}

\usepackage{amsfonts}
\usepackage{amsmath}
\usepackage[psamsfonts]{amssymb}
\usepackage{amstext}
\usepackage{amsthm}
\usepackage{latexsym}
\usepackage{graphicx}
\usepackage{subfig}

\usepackage{enumerate}
\usepackage{algorithm}
\usepackage{algorithmic}
\usepackage{url}
\usepackage{dsfont}
\usepackage[mathscr]{euscript}
\usepackage{booktabs}
\usepackage{makecell}

\usepackage{tabularx}
\usepackage{xcolor, colortbl}

\definecolor{Gray}{gray}{0.85}
\newcolumntype{g}{>{\columncolor{Gray}}c}

\makeatletter
\newtheorem*{rep@theorem}{\rep@title}
\newcommand{\newreptheorem}[2]{%
\newenvironment{rep#1}[1]{%
 \def\rep@title{#2 \ref{##1}}%
 \begin{rep@theorem}}%
 {\end{rep@theorem}}}
\makeatother

\hypersetup{
  colorlinks   = true, 
  urlcolor     = blue, 
  linkcolor    = blue, 
  citecolor   = blue 
}

\def\Rset{\mathbb{R}}

\DeclareMathOperator*{\E}{\mathbb{E}}

\DeclareMathOperator*{\argmin}{\rm argmin}

\newcommand{\set}[1]{\{#1\}}

\newcommand{\h}{\widehat}

\newcommand{\cD}{\mathcal{D}}
\newcommand{\cL}{\mathcal{L}}

\newcommand{\cX}{\mathcal{X}}
\newcommand{\cY}{\mathcal{Y}}

\newcommand{\e}{\epsilon}

\newcommand{\sD}{{\mathscr D}}
\newcommand{\sU}{{\mathscr U}}
\newcommand{\sfD}{{\mathsf D}}
\newcommand{\sfd}{{\mathsf d}}
\newcommand{\msD}{{\mathscr D}^{1}} 
\newcommand{\msU}{{\mathscr U}^{1}} 

\newtheorem{theorem}{Theorem}
\newreptheorem{theorem}{Theorem}
\newtheorem{lemma}[theorem]{Lemma}
\newreptheorem{lemma}{Lemma}

\newreptheorem{corollary}{Corollary}
\newtheorem{proposition}[theorem]{Proposition}
\newreptheorem{proposition}{Proposition}
\newcommand{\ignore}[1]{}

\begin{document}

\twocolumn[

\aistatstitle{Multiple-Source Adaptation for Regression Problems}
\aistatsauthor{ Judy Hoffman \And Mehryar Mohri \And  Ningshan Zhang }
\aistatsaddress{ EECS, UC Berkeley \And  Courant Institute, NYU and Google Research \And Courant Institute, NYU } ]

\begin{abstract}
  We present a detailed theoretical analysis of the problem of
  multiple-source adaptation in the general stochastic scenario,
  extending known results that assume a single target
  labeling function. Our results cover a more realistic scenario and
  show the existence of a single robust predictor accurate for
  \emph{any} target mixture of the source distributions.
  Moreover, we present an efficient and practical optimization solution 
  to determine the robust predictor in the important case of squared loss,
  by casting the problem as an instance of DC-programming.
  We report the results of experiments
  with both an artificial task and a sentiment analysis task.  We
  find that our algorithm outperforms competing approaches by
  producing a single robust model that performs well on any target
  mixture distribution.
\end{abstract}

\section{Introduction}
\label{sec:intro}

In many modern applications, often the learner has access to
information about several source domains, including accurate
predictors possibly trained and made available by others, but no
direct information about a target domain for which one wishes to
achieve a good performance. The target domain can typically be viewed
as a combination of the source domains, that is a mixture of their
joint distributions, or it may be close to such mixtures.

Such problems arise commonly in speech recognition where 
different groups of speakers (domains) yield different acoustic
models and the problem 
is to derive an accurate
acoustic model for a broader population that may be viewed as a
mixture of the source groups~\cite{liao_icassp13}.
In visual recognition multiple
image databases exist each with its own bias and labeled
categories~\cite{efros_cvpr11}, but the target application may contain
some combination of image types and any category may need to be
recognized. 
Finally, in sentiment analysis accurate
predictors may be available for sub-domains such as TVs, laptops and
CD players for which labeled training data was at the learner's
disposal, but not for the more general category of electronics, which
can be modeled as a mixture of the sub-domains~\citep{blitzer_acl07,
  dredze_nips08}.

Many have considered the case of transferring between a single source
and known target domain either through unsupervised adaptation (no
labels in the target) techniques~\citep{gong_cvpr12, long_icml15,
  ganin_icml15, tzeng_iccv15} or supervised (some labels in the
target)~\citep{saenko_eccv10, yang_acmm07, hoffman_iclr13, rcnn}.

Here, we focus on the case of multiple source domains and ask how the
learner can combine relatively accurate predictors available for
source domains to derive an accurate predictor for \emph{any} new
mixture target domain. This is known as the \emph{multiple-source
  adaption problem} first formalized and analyzed theoretically by
\citet{MansourMohriRostamizadeh2008, MansourMohriRostamizadeh2009} and
later studied for various applications such as object
recognition~\citep{hoffman_eccv12, gong_iccv13, gong_nips13}.  In the
most general setting, even unlabeled information may not be available
about the target domain, though one can expect it to be a mixture of
the source distributions. Thus, the problem is also closely related to
domain generalization
\citep{pan_tkda2010,MuandetBalduzziScholkopf2013,xu_eccv14}, where
knowledge from an arbitrary number of related domains is combined to
perform well on a previously unseen domain.  Several algorithms have
been proposed based on convex combinations of predictors from multiple
source domains \citep{schweikert2009empirical,
  chattopadhyay2012multisource, duan2012exploiting, duan2012domain}.

We build on the work of~\citet{MansourMohriRostamizadeh2008}, which
proved that a distribution weighted predictor admits a small loss on
any target mixture for the special case of deterministic predictors,
however the algorithmic problem of finding such a weighting was left
as a difficult question since it required solving a Brouwer
fixed-point problem.

The first main contribution of this work is extending the
multiple-source adaptation theory from deterministic true predictors
to stochastic ones, that is the scenario where there is a distribution
over the joint feature and label space, $\cX \times \cY$, as opposed
to deterministic case where a unique target labeling function is
assumed. This generalization is needed to cover the realistic cases in
applications.  Under the assumption that the conditional distributions
are the same for all domains, our analysis shows that for the
stochastic scenario there exists a robust predictor that admits a
small expected loss with respect to \emph{any} mixture distribution.
We further extend this result to an arbitrary distribution with small
R\'enyi divergence with respect to the family of mixtures.  We also
extend it to the case where, instead of having access to the ideal
distributions, only estimate distributions are used for deriving that
hypothesis. Finally, we present a novel extension of the theory not
covered by
\citet{MansourMohriRostamizadeh2008,MansourMohriRostamizadeh2009}, to
the more realistic case where the conditional distributions are
distinct.

Our second main contribution is a new formulation of the
problem of finding that robust predictor.  Note again that no
algorithm was previously given by \citet{MansourMohriRostamizadeh2008}
for this problem (for their experiments they directly used the target mixture
weights).  We show that, in the
important case of the squared loss, the problem can be cast as a
DC-programming problem and we present an efficient and practical
optimization solution for it. We have fully implemented our algorithm
and report the results of experiments with both an artificial task and
a sentiment analysis task.  We find that our algorithm outperforms
competing approaches by producing a single robust model that performs
well on any target mixture distribution.

\ignore{ We first studied our DC-programming algorithm by finding the
optimal mixture weighting when the source distributions are known
using synthetic examples. We verified that, in practice, our
DC-programming solution often yields the global optimum. Next, we
evaluated our domain generalization approach There, }

\section{Problem set-up}
\label{sec:setup}
We consider a multiple-source domain adaptation problem in the general
stochastic scenario where there is a distribution over $\cX \times
\cY$, as opposed to the special case where a target function mapping
from $\cX$ to $\cY$ is assumed (deterministic
scenario)~\citep{MansourMohriRostamizadeh2008}. This extension is
needed for the analysis of the most common and realistic learning
set-ups in practice.

Let $\cX$ denote the input space and $\cY$ the output space.  We will
identify a \emph{domain} with a distribution over $\cX \times \cY$ and
consider the scenario where the learner has access to a predictor
$h_k \colon \cX \to \cY$, for each domain $\sD_k$, $k = 1, \ldots, p$.
The learner's objective is to combine these predictors so as to
minimize error for a target domain that may be mixture of the source
domains, or close to such mixtures.

In order to find such a good combination, the learner also needs
access to $\sD_k(x,y)$ or a good estimate $\h \sD_k(x,y)$. In fact, in
practice the learner would rarely have access to $\sD_k$, but only
$\h \sD_k$.  Note that, even though $h^*(x) = \sum_y y \sD_k(x, y)$ is
the optimal predictor under the squared loss for domain $k$,
$\sum_y y \h \sD_k(x, y)$ may have a poor generalization performance
and not form a good predictor.

Much of our theory applies to an arbitrary loss function $L\colon \cY
\times \cY \to \Rset_+$ that is convex and continuous. We will denote
by $\cL(\sD, h)$ the expected loss of a predictor $h\colon \cX \to
\cY$ with respect to the distribution $\sD$:
\begin{eqnarray}
  \cL(\sD, h) &=& \E_{(x, y) \sim \sD} \big[L( h(x), y) \big]\nonumber \\
  &=& \sum_{(x, y) \in \cX \times \cY} \sD(x, y) L(h(x), y).
\end{eqnarray}
But, we will be particularly interested in the case where $L$ is the
squared loss, that is $L(h(x), y) = (h(x) - y)^2$.

We will assume that each $h_k$ is a relatively accurate predictor for
the distribution $\sD_k$ and that there exists $\e > 0$ such that
$\cL(\sD_k, h_k) \leq \e$ for all $k \in [1, p]$. We will first assume that
the conditional probability of the output labels is the same for all
source domains, that is, for any $(x, y)$, $\sD_k(y | x)$ does not
depend on $k$. This is a natural assumption that is more general than
the one adopted by \citet{MansourMohriRostamizadeh2008} in the
analysis of the deterministic scenario where exactly the same labeling
function $f$ is assumed for all source domains. 
We then relax this requirement later and provide guarantees for distinct conditional probabilities.

We will also assume that the average loss of the source predictors
under the uniform distribution $\msU$ over $\cX$ is bounded, that is,
there exists $M \geq 0$ such that
\begin{equation*} \frac{1}{p} \sum_{k = 1}^p \sum_{(x, y)\in \cX \times \cY} \cD_k(y |
x) \msU(x) L(h_k(x), y) \leq M.
\end{equation*} 

\ignore{ This is a mild assumption that is satisfied
in particular when there exists $\mu > 0$ such that $h_k(x, y) \geq
\mu$ for all $k \in [1, p]$ and $(x, y) \in \cX \times \cY$.  }

\section{Theoretical analysis}
\label{sec:theory_analysis}

In this section, we present a theoretical analysis of the general
multiple-source adaptation in stochastic case.

We first show that when the conditional distributions $\sD_k(y|x)$
are the same for all source domains, there exists a single weighted
combination rule $h_z^\eta$ that admits a small loss with respect to
\emph{any} target mixture distribution $\sD_\lambda$, that is any
$\lambda$ (Section~\ref{sec:mixturedistribution}). We further give
guarantees for the loss of that hypothesis with respect to an
arbitrary distribution $\sD_T$
(Section~\ref{sec:arbitrarydistribution}).  Next, we extend our
guarantees to the case where the source distributions $\sD_k$ are not
directly accessible and where the weighted combination rule is derived
using estimates of the distributions $\sD_k$
(Section~\ref{sec:estimatedistributions}). Finally, we provide
guarantees for the most general case where the conditional 
distributions $\sD_k(y|x)$ are distinct and for an arbitrary
target distribution $\sD_T$ (Section~\ref{sec:distinctconditionals}).

\subsection{Mixture target distribution}
\label{sec:mixturedistribution}

Here we consider the case of a target distribution, that is a mixture
distribution $\sD_\lambda = \sum_{k = 1}^p \lambda_k \sD_k$, for some
$\lambda = (\lambda_1, \ldots, \lambda_p)$ in the simplex $\Delta =
\set{\lambda \in \Rset^p \colon \forall k \in [p], \lambda_k \geq 0,
\sum_{k = 1}^p \lambda_k = 1}$.  The mixture weight $\lambda \in
\Delta$ defining $\sD$ is not known to us. Thus, our objective will be
to find a hypothesis with small loss with respect to \emph{any}
mixture distribution $\sD_\lambda$, using a combination of trained
domain-specific hypotheses $h_k$s.

Extending the definition given by
\citet{MansourMohriRostamizadeh2008}, we define the
distribution-weighted combination of the models $h_k$, $k \in [1, p]$
as follows. For any $z \in \Delta$, $\eta > 0$, and $x \in \cX$,
\begin{align}
\label{eq:h}
h_z^\eta(x) 
& = \mspace{-5mu}  \sum_{k = 1}^p \sum_{y \in \cY} \frac{z_k \sD_k(x, y) + \eta  \frac{\sU(x, y)}{p}}{ \sum_{k =
  1}^p  \sum_{y \in \cY} z_k \sD_k(x, y) + \eta \sU(x,y)} h_k(x) \\
& = \sum_{k = 1}^p \frac{z_k \msD_k(x) + \eta  \frac{\msU(x)}{p}}{\sum_{k = 1}^p z_k \msD_k(x) + \eta \msU(x)} h_k(x),
\end{align}
where $\sU$ is the uniform distribution over $\cX \times \cY$,
and where, 
we denote by $\msD(x)$ the marginal distribution over $\cX$:
$\msD(x) = \sum_{y \in \cY} \sD(x, y)$.

The distribution weighted combination is the natural ensemble 
solution in the case where the models $h_k$s have been trained 
with different distributions. In fact, when $\sD_k$s coincide,
$h_z$ is the standard convex ensemble of the $h_k$s.

Observe that for any $x \in \cX$, $h^\eta_z(x)$ is continuous in $z$
since the denominator in \eqref{eq:h} is positive ($\eta > 0$). By the
continuity of $L$, this implies that, for any distribution $\sD$,
$\mathcal{L}(\sD, h_z^{\eta})$ is continuous in $z$.

Our proof makes use of the following Fixed Point Theorem of Brouwer.
\begin{theorem}
  For any compact and convex non-empty set $C \subset \Rset^p$ and any
continuous function $f\colon C\rightarrow C$, there is a point $x \in
C$ such that $f(x) = x$.
\end{theorem}

\begin{lemma}
\label{lemma:brouwer} 
For any $\eta, \eta' > 0$, there exists $z \in \Delta$, with $z_k \neq
0$ for all $k \in [1, p]$, such that the following holds for the
distribution weighted combining rule $h_z^\eta$:
\begin{equation}
\label{eq:property}
\forall k \in [p], \quad \cL(\sD_k, h_z^\eta) \leq \sum_{j = 1}^p z_j \cL(\sD_j, h_z^\eta)  + \eta'.
\end{equation}
\end{lemma}
\begin{proof}
Consider the mapping $\Phi \colon
  \Delta \to \Delta$ defined for all $z \in \Delta$ by
\begin{equation*} 
[\Phi(z)]_k = \frac{z_k \, \cL(\sD_k, h_z^\eta) +
      \frac{\eta'}{p} }{\sum_{j = 1}^p z_j \cL(\sD_j, h_z^\eta) + \eta'}.
\end{equation*}
$\Phi$ is continuous since $\cL(\sD_k, h_z^\eta)$ is a continuous
function of $z$ and since the denominator is positive ($\eta' > 0$).
Thus, by Brouwer's Fixed Point Theorem, there exists $z \in \Delta$
such that $\Phi(z) = z$. For that $z$, we can write
\begin{equation*}
z_k = \frac{z_k \, \cL(\sD_k, h_z^\eta) +
      \frac{\eta'}{p} }{\sum_{j = 1}^p z_j \cL(\sD_j, h_z^\eta) + \eta'},
\end{equation*}
for all $k \in [1, p]$.  Since $\eta'$ is positive, we must have $z_k
\neq 0$ for all $k$. Dividing both sides by $z_k$ gives $\cL(\sD_k,
h_z^\eta) = \sum_{j = 1}^p z_j \cL(\sD_j, h_z^\eta) + \eta' -
\frac{\eta'}{p z_k} \leq \sum_{j = 1}^p z_j \cL(\sD_j, h_z^\eta) +
\eta'$, which completes the proof.
\end{proof}

\begin{theorem}
\label{th:mixture}
For any $\delta > 0$, there exists $\eta > 0$ and $z\in \Delta$, such
that $\cL(\sD_\lambda,h_z^\eta)\leq \e +\delta$ \ for any mixture
parameter $\lambda \in \Delta$.
\end{theorem}
The proof uses the convexity of the loss function and
Lemma~\ref{lemma:brouwer}.  The full proof is given in the
supplementary. 
The theorem shows the existence of a mixture weight $z \in \Delta$ and
$\eta > 0$ with a remarkable property: for any $\delta > 0$,
regardless of which mixture weight $\lambda \in \Delta$ defines the
target distribution, the loss of $h_z^\eta$ is at most $\e + \delta$,
that is arbitrarily close to $\e$. $h_z^\eta$ is therefore a
\emph{robust} hypothesis with a favorable property for any mixture
target distribution.
More precisely, by the proof of the theorem, for any $z \in \Delta$
verifying \eqref{eq:property}, $h_z^\eta$ admits this property. We
exploit this in the next section to devise an algorithm for finding
such a $z \in \Delta$.

\subsection{Arbitrary target distribution}
\label{sec:arbitrarydistribution}

Here, we extend the results of the previous section to the case of an
arbitrary target distribution $\sD_T$ that may not be a mixture of the
source distributions, by extending the results of
\citet*{MansourMohriRostamizadeh2009}.

We will assume that the loss of the source hypotheses $h_k$ is
bounded, that is $L(h_k(x), y) \leq M$ for all $(x, y) \in \cX \times
\cY$. By convexity, this immediately implies that for any distribution
combination hypothesis $h_z^\eta$,
\vspace{-.2cm}%
\begin{align*}
L(h_z^\eta(x), y) 
& \leq \sum_{k=1}^p \frac{z_k
    \msD_k(x) + \eta  \frac{\msU(x)}{p}}{\msD_z(x) + \eta \msU(x)} L\big(h_k(x), y \big)\\
& \leq \sum_{k=1}^p \frac{z_k
    \msD_k(x) + \eta  \frac{\msU(x)}{p}}{\msD_z(x) + \eta \msU(x)} M = M.
\end{align*}
Our extension to an arbitrary target distribution $\sD_T$ is based on
the divergence of $\sD_T$ from the family of all mixtures of the
source distributions $\sD_k$, $k \in [p]$. Different divergence
measures could be used in this context.  The one that naturally comes
up in our analysis as in previous work, is the \emph{R\'enyi
  Divergence} \citep{Renyi1961}.  The R\'enyi Divergence is
parameterized by $\alpha$ and denote by $\sfD_\alpha$. The
$\alpha$-R\'enyi Divergence of two distributions $\sD$ and $\sD'$ is
defined by
\begin{equation}
  \sfD_\alpha(\sD \parallel \sD') = \frac{1}{\alpha - 1}
  \log \mspace{-20mu} 
  \sum_{(x, y) \in \cX \times \cY} \mspace{-25mu} 
  \sD(x, y)
  \left[ \frac{\sD(x, y)}{\sD'(x, y)} \right]^{\alpha - 1} \mspace{-25mu}.
\end{equation}
It can be shown that the R\'enyi Divergence is always non-negative and
that for any $\alpha > 0$, $\sfD_\alpha(\sD \parallel \sD') = 0$ iff
$\sD = \sD'$, (see \cite{arndt}).  The R\'enyi divergence coincides
with the following known measure for some specific values of $\alpha$:
\begin{itemize}

\item $\alpha = 1$: $\sfD_1(\sD \parallel \sD')$ coincides with the
  standard relative entropy or KL-divergence.

\item $\alpha = 2$:
  $\sfD_2(\sD \parallel \sD') = \log \E_{(x, y) \sim \sD} \big[
  \frac{\sD(x, y)}{\sD'(x, y)} \big]$
  is the logarithm of the expected probabilities ratio.

\item $\alpha = +\infty$:
  $\sfD_\infty(\sD \parallel \sD') = \log \displaystyle \sup_{(x, y) \in \cX \times
    \cY} \textstyle \big[ \frac{\sD(x, y)}{\sD'(x, y)} \big]$,
  which bounds the maximum ratio between two probability
  distributions.
\end{itemize}

We will denote by $\sfd_\alpha(\sD \parallel \sD')$ the exponential:
\begin{equation}
  \mspace{-1mu}
\sfd_\alpha(\sD \parallel \sD') = e^{\sfD_\alpha(\sD \parallel \sD')}\mspace{-6mu} =
  \mspace{-4mu} \left[\mspace{-2mu} \sum_{(x, y) \in \cX \times \cY} \frac{\sD^\alpha(x, y)}{\sD'^{\alpha - 1}(x, y)}
  \right]^{\frac{1}{\alpha - 1}} \mspace{-35mu}. \mspace{-15mu}
\end{equation}
Given a class of distributions $\cD$, we denote by
$\sfD_\alpha(\sD \parallel {\cD})$ the infimum
$\inf_{\sD'\in {\cD}} \sfD_\alpha(\sD \parallel \sD')$.  We will
concentrate on the case where $\cD$ is the class of all mixture
distributions over a set of $k$ source distributions, i.e.,
${\cD} = \set{\sD_\lambda \colon \sD_\lambda = \sum_{k = 1}^p
  \lambda_k \sD_k, \lambda \in \Delta}$.

\begin{theorem}
\label{th:arbitrary}
Let $\sD_T$ be an arbitrary target distribution.
For any $\delta > 0$, there exists
$\eta > 0$ and $z \in \Delta$, such that
the following inequality holds for any $\alpha > 1$:
\begin{equation*}
\cL(\sD_T, h_z^\eta) 
\leq \Big[(\e + \delta) \, \sfd_\alpha(\sD_T \parallel \cD)
\Big]^{\frac{\alpha - 1}{\alpha}} M^{\frac{1}{\alpha }}.
\end{equation*}
\end{theorem}

The proof of Theorem~\ref{th:arbitrary} is given in 
the supplementary. 
Note that in the particular case where $\sD_T$ is in $\cD$, we have
$\sfd_\alpha(\sD_T \parallel \cD) = 1$. For $\alpha \to +\infty$, we
then retrieve the result of Theorem~\ref{th:mixture}.

\subsection{Arbitrary target distribution and estimate distributions}
\label{sec:estimatedistributions}

We now further extend our analysis to the case where the distributions
$\sD_k$ are not directly available to the learner and where instead
estimates $\h \sD_k$ have been derived from data.

For $k \in [p]$, let $\h \sD_k$ be an estimate of $\sD_k$ and define
$\h \e$ by
\begin{equation}
\label{eq:he}
\h \e = \max_{k \in [p]} \Big[\e  \, \sfd_\alpha(\h \sD_k \parallel \sD_k)
\Big]^{\frac{\alpha - 1}{\alpha}} M^{\frac{1}{\alpha }}.
\end{equation}
Note that when for all $k \in [p]$, $\h \sD_k$ is a good estimate of
$\sD_k$, then $\sfd_\alpha(\h \sD_k \parallel \sD_k)$ is close
to one and, for $\alpha \to +\infty$, $\h \e$ is very close to $\e$.
We will denote by $\h \cD$ the family of mixtures of the
estimates $\h \cD_k$:
\begin{equation}
\label{eq:hD}
\h \cD = \left\{ \sum_{k = 1}^p \lambda_k \h \sD_k\colon k \in \Delta \right\}.
\end{equation}
We will denote by $\h h_z^\eta$ the distribution weighted combination
hypothesis based on the estimate distributions $\h \sD_k$:
\begin{equation}
\label{eq:hh}
  \h h_z^\eta(x) = \sum_{k=1}^p\frac{z_k \h \msD_k(x) + \eta \, \frac{\msU(x)}{p}}{
    \sum_{j = 1}^p z_j \h \msD_j(x) + \eta \msU(x) } h_k(x),
\end{equation}
where we denote by $\h \msD(x)$ the marginal distribution over $\cX$: 
$\h \msD(x) = \sum_{y\in \cY} \h \sD(x,y)$.

\begin{theorem}
\label{th:estimate}
Let $\sD_T$ be an arbitrary target distribution.
Then, for any $\delta > 0$, there exists $\eta > 0$ and
$z \in \Delta$, such that the following inequality holds for any
$\alpha > 1$:
\begin{equation*}
\cL(\sD_T, \h h_z^\eta) 
\leq \Big[(\h \e + \delta) \, \sfd_\alpha(\sD_T \parallel \h \cD)
\Big]^{\frac{\alpha - 1}{\alpha}} M^{\frac{1}{\alpha }}.
\end{equation*}
\end{theorem}

The proof of Theorem~\ref{th:estimate} depends heavily on
the result of Theorem~\ref{th:arbitrary}, and is given in 
the supplementary. 
This result shows that there exists a predictor $\h h_z^\eta$ based on the estimate
distributions $\h \sD_k$ that is $\h \e$-accurate with respect to
any target distribution $\sD_T$ whose R\'enyi divergence with
respect to the family $\h \cD$ is not too large, i.e.\
$\sfd_\alpha(\sD_T \parallel \h \cD)$ close to $1$.  Furthermore,
$\h \e$ is close to $\e$, provided that $\h \sD_k$s are good estimates
of $\sD_k$s, i.e.\ $\sfd_\alpha(\h \sD_k \parallel \sD_k)$ close to
$1$.

Theorem~\ref{th:estimate} used the R\'enyi divergence in both
directions: $\sfd_\alpha(\sD_T \parallel \h \cD)$ requires
$\text{Supp}(\sD_T) \subseteq \text{Supp}(\h\cD)$, and
$\sfd_\alpha(\h \sD_k \parallel \sD_k)$ requires
$\text{Supp}(\h\sD_k) \subseteq \text{Supp}(\sD_k), k\in [p]$.  Many
density estimation methods fulfill these requirements, for instance
kernel density estimation, Maxent models, or $n$-gram language
models. In our experiments (Section~\ref{sec:eval}), we fit our data
using a bigram language model for sentiment analysis.

For any two distributions $\sD, \sD'$, the R\'enyi divergence (and
$\sfd_\alpha(\sD \parallel \sD')$) is nondecreasing as a function of
$\alpha$, and
\begin{equation*}
  \sfd_\alpha(\sD \parallel \sD') \le \sfd_\infty(\sD \parallel \sD') = 
  \sup_{(x,y)\in\cX\times\cY} \left[\frac{\sD(x,y)}{\sD'(x,y)}\right].
\end{equation*}
Thus, it is reasonable to assume that
$\sfd_\alpha(\sD_T \parallel \h \cD)$ and
$\sfd_\alpha(\h \sD_k \parallel \sD_k), k \in [p]$ are not too large
when $\h \sD_k$ are good estimates of $\sD_k, k \in [p]$.

\subsection{Arbitrary target distribution and distinct conditional probability distributions}
\label{sec:distinctconditionals}

This section examines the case where the conditional probability
distributions of the source domains are distinct, and the target
distribution is arbitrary. This is a novel extension that was not
discussed in \citet{MansourMohriRostamizadeh2009}.

Let $\sD_T(\cdot |x)$ denote the conditional 
probability distribution on target domain,
and $\sD_k(\cdot|x)$ denote 
the conditional probability distribution associated to source $k$. 
$\sD_T(\cdot |x), \sD_k(\cdot|x),k\in [p]$ are not necessarily the same.
Define $\e_{T}$ by
\begin{equation*}
  \e_T \mspace{-3mu} = \mspace{-3mu} \max_{k\in [p]} \Big[\mathbb{E}_{\msD_k(x)} \sfd_{\alpha}\left(
    \sD_T(\cdot | x) \parallel \sD_k(\cdot | x)\right)^{\alpha-1}\Big]^{\mspace{-4mu}\frac{1}{\alpha}}
  \mspace{-5mu} \e^{\frac{\alpha-1}{\alpha}} \mspace{-5mu}  M^{\frac{1}{\alpha}}.
\end{equation*}
Therefore, if for every domain $k$, $\sD_T(\cdot | x)$ is on average
not too far away from $\sD_k(\cdot | x)$, then, for $\alpha$ large,
$\e_T$ is close to $\e$.  Let $\sD_{k,T}(x,y)=\msD_k(x)\sD_T(y|x)$,
and $\sD_{P,T}$ be the class of mixtures of $\sD_{k,T}$:
\begin{equation*}
\sD_{P,T} = \left \{ \sum_{k=1}^p \lambda_k \sD_{k,T}, k\in \Delta \right \}.
\end{equation*}

\begin{theorem}
Let $\sD_T$ be an arbitrary target distribution. Then, for any $\delta >0$, 
there exists $\eta >0$ and $z \in \Delta$ such that the following inequality
holds for any $\alpha>1$:
\begin{equation*}
\cL(\sD_T, h_z^\eta) \le \Big[ (\e_{T} + \delta) \sfd_\alpha (\sD_T \parallel \sD_{P,T})\Big]^
{\frac{\alpha-1}{\alpha}} M^{\frac{1}{\alpha}}.
\end{equation*}
\end{theorem}
\begin{proof}
For any domain $k$, by H\"older's inequality, the following holds:
\begin{align*}
&\cL(\sD_{k,T},h_k) = \sum_{x,y} \msD_k(x) \sD_T(y|x) L_k(x,y) \\
=&\sum_{x}\msD_k(x) \sum_y\left[\frac{\sD_T(y|x)}{\sD_k(y|x)^
{\frac{\alpha-1}{\alpha}}}\right]
\left[\sD_k(y|x)^{\frac{\alpha-1}{\alpha}}L_k(x,y)\right] \\
 \le& \sum_{x}\msD_k(x)  \sfd_{\alpha}(x;T,k)^{\frac{\alpha-1}{\alpha}} 
\Big[\sum_{y}\sD_k(y| x)L_k(x,y)^{\frac{\alpha}{\alpha-1}}\Big]^{\frac{\alpha-1}{\alpha}}
\end{align*}
where, for simplicity, we write 
$\sfd_{\alpha}(x;T,k)  = 
 \sfd_{\alpha}\left(\sD_T(\cdot| x)\parallel \sD_k(\cdot| x)\right)$,
 and $L_k(x,y)=L(h_k(x),y)$.
 Using the fact that the loss is bounded and H\"older's inequality again, 
  \begin{align*}
&\cL(\sD_{k,T},h_k) \\
 &\mspace{-10mu} \le 
  \sum_{x}\mspace{-2mu} \msD_k(x)^{\frac{1}{\alpha}}
\sfd_{\alpha}(x;T,k)^{\frac{\alpha-1}{\alpha}} 
\mspace{-5mu}\left[\sum_{y}\mspace{-2mu} 
\sD_k(x,y)L_k(x,y)\right]^{\frac{\alpha-1}{\alpha}}
 \mspace{-10mu} M^{\frac{1}{\alpha}}\\
&\mspace{-10mu} \le \left[\sum_{x}\msD_k(x)\sfd_{\alpha}(x;T,k)^{\alpha-1}\right]^{\frac{1}{\alpha}}
\cL(\sD_k,h_k)^{\frac{\alpha-1}{\alpha}} M^{\frac{1}{\alpha}}  \\
&\mspace{-10mu} \le  \Big[\mathbb{E}_{\msD_k} \sfd_{\alpha}(x;T,k)^{\alpha-1}\Big]^{\frac{1}{\alpha}}
\e^{\frac{\alpha-1}{\alpha}}  M^{\frac{1}{\alpha}}  \le \e_{T}.
\end{align*}
We can now apply the result of Theorem~\ref{th:arbitrary}, 
with $\e_{T}$ instead of $\e$ and $\sD_{k,T}$ instead of $\sD_k$. 
This completes the proof.
 \end{proof}

\section{Algorithm}
\label{sec:algorithm}

In the previous section, we showed that there exists a vector $z$
defining a distribution-weighted combination hypothesis $h_z^\eta$
that admits very favorable properties. But can we find that vector $z$
efficiently? This is a key question in the learning problem of
multiple-source adaptation which was not discussed by
\citet*{MansourMohriRostamizadeh2008,MansourMohriRostamizadeh2009}.
No algorithm was previously reported to determine the mixture
parameter $z$ (even in the deterministic scenario).

Since $z \in \Delta$ is the solution of Brouwer's Fixed Point Theorem,
one general approach consists of using the combinatorial algorithm of
\citet*{Scarf1967} based on simplicial partitions, which makes use of
a result similar to Sperner's Lemma \cite{Kuhn1968}. However, that
algorithm is costly in practice and its computational complexity is
exponential in $p$.  Other algorithms were later given by
\citet{Eaves1972} and \citet{Merrill1972}. But, it has been shown more
generally by \citet*{HirschPapadimitriouVavasis1989} (see also
\citep{ChenDeng2008}) that any general algorithm for computing
Brouwer's Fixed Point based on function evaluations must in the worst
case make a number of evaluation calls that is exponential both in $p$
and the number of digits of approximation accuracy.

Several heuristics can be attempted to find the solution. One option
is to resort to a brute-forced gradient descent technique, which does
not benefit from any convergence guarantee since the problem is not
convex. 
Another approach consists of
repeatedly applying the function $\Phi$ to $z$ with the hope of
approaching its fixed point. But the starting point must lie in the
attractive basin of $\Phi$, which cannot be guaranteed in general.

In this section, we give a practical and efficient algorithm for
finding the vector $z$.  We first show that $z$ is the solution of a
general optimization problem. Next, by using the differentiability of
the loss we show that the optimization problem can be cast as a
DC-programming problem.  This leads to an efficient algorithm that is
guaranteed to converge to a stationary point.  Additionally, we show that
it is straightforward to test if the solution obtained is the global
optimum. While we are not proving that the local stationary point found by our
algorithm is the global optimum, empirically, we observe that
that is indeed the case.  Note that the global minimum of our
DC-programming problem can also be found using a cutting plane method
of \citet{HorstThoai1999} that does not admit known algorithmic
convergence guarantees or a branch-and-bound algorithm with
exponential convergence \citep{HorstThoai1999}.

\subsection{Optimization problem}

Theorem~\ref{th:mixture} shows that the hypothesis $h_z^\eta$ based
on the mixture parameter $z$ of Lemma~\ref{lemma:brouwer} benefits
from a strong generalization guarantee. Thus, our problem consists
of finding a parameter $z$ verifying the statement of 
Lemma~\ref{lemma:brouwer}, that is, 
\begin{equation}
\label{eq:ineq}
\cL(\sD_k, h_z^\eta)\leq \cL(\sD_z, h_z^\eta) + \eta',
\end{equation}
for any $k \in [1,p]$. This in turn can be formulated as a min-max problem
\begin{equation*}
\min_{z \in \Delta} \max_{k \in [1,p]} \cL(\sD_k, h_z^\eta) - \cL(\sD_z, h_z^\eta),
\end{equation*}
which can be equivalently formulated as the following optimization problem:
\begin{align}
\label{eq:opt}
\min_{z \in \Delta, \gamma \in \Rset} & \ \gamma\\
\text{s.t.} & \ \cL(\sD_k, h_z^\eta) - \cL(\sD_z, h_z^\eta) \leq
\gamma, \quad \forall k \in [1, p].\nonumber
\end{align}
Note that, same as theorem~\ref{th:mixture} we assume the conditional
probabilities are same across domains.  

\subsection{DC-Programming}
\label{sec:dc}

In this section, we show that Problem~\ref{eq:opt} can be cast as a
DC-programming (difference of convex programming) problem, which can
be tackled by several algorithms designed for this class of problems. We first rewrite $h^\eta_z$ in terms of two affine functions
of $z$, 
$J_z$ and $K_z$:
\begin{align*}
  h^\eta_z(x) 
& = \sum_{k = 1}^p\frac{z_k \msD_k(x) + \frac{\eta}{p} \msU(x)}{
    \sum_{j = 1}^p z_j \msD_j(x) + \eta \, \msU(x)} h_k(x) 
 = \frac{J_z(x)}{K_z(x)},
\end{align*}
where, for any $x \in \cX$,
\begin{align*}
J_z(x) & =\sum_{k = 1}^p z_k \msD_k(x) h_k(x) + \eta \msU(x) H(x)\\
H(x) &= \frac{1}{p}  \sum_{k = 1}^p h_k(x)\\
K_z(x) & = \msD_z(x)
  + \eta \, \msU(x).
\end{align*}

We will assume that the loss of the source hypotheses $h_k$ is
bounded, that is $L(h_k(x), y) \leq M$ for all $(x, y) \in \cX \times
\cY$. It follows that 
$L(h_z^\eta(x), y) \leq M$ for all $(x, y) \in \cX \times
\cY$ and $z \in \Delta$.

\begin{proposition}
\label{prop:hzdc}
The following
decomposition holds for all $(x,y) \in \cX \times \cY$, 
\begin{equation*}
(h_z^\eta(x) -y)^2  = f_z(x,y)-g_z(x),
\end{equation*}
where for every $(x,y)\in \cX \times \cY$, 
$f_z$ and $g_z$ are convex functions defined for all $z$:
\begin{align*}
f_z(x,y) & = \left(h_{z}^\eta(x)-y\right)^2-2M\log K_z(x),\\
g_z(x) &= -2M\log K_z(x).
\end{align*}
\end{proposition}

\begin{proof}
We can write the Hessian matrix of $f_z$ and $g_z$ as
\begin{align*}
H_{f_z} &=  \frac{2}{K_z^2} \left[h_{D,z} h_{D,z}^T + \left(M - (y-h_z^\eta)^2\right) DD^T \right]  \\
H_{g_z} &=  \frac{2M}{K_z^2}  DD^T  
\end{align*}
where $h_{D,z}$ is a $p$-dimensional vector defined as
 $[h_{D,z}]_k = \sD_k(h_k + y - 2 h_z^\eta)$ for $k\in [p]$,
and $D = (\sD_1, \sD_2, \dots,\sD_p)^T$. 
Using the fact that $ M \ge (y-h_z^\eta)^2$, $H_{f_z}$ and $H_{g_z}$ are 
positive semidefinite matrices, therefore $f_z, g_z$ are convex functions of $z$.
\end{proof}

\begin{proposition}
For any $k \in [p]$, the following decomposition holds
\begin{equation*}
\cL(\sD_k, h_z^\eta) - \cL(\sD_z, h_z^\eta) = u_k(z) - v_k(z),
\end{equation*}
where $u_k$ and $v_k$ are the convex functions defined for all $z$ by
\begin{align*}
u_k(z) &=\cL(\sD_k+\eta\msU \sD_k(\cdot | x),h_z^\eta)-2M\mathbb{E}_{\msD_k+\eta\msU}\log K_z,\\
v_k(z) &=\cL(\sD_z+\eta\msU \sD_k(\cdot | x),h_z^\eta)-2M\mathbb{E}_{\msD_k+\eta\msU}\log K_z.
\end{align*}
\end{proposition}
\begin{proof}
By proposition~\ref{prop:hzdc}, 
$u_k(z)= \sum_{(x,y)} (\msD_k+\eta \msU)(x)\sD_k(y|x) f_z(x,y)$
is convex. Similarly, we can write the second term of $v_k(z)$ as 
$\sum_x (\msD_k + \eta \msU)(x) g_z(x)$, it is convex.
Using the notation previously defined, we can write the first term of $v_k(z)$ as
\begin{align*}
\lefteqn{\cL(\sD_z+\eta\msU \sD_k(\cdot | x), h_z^\eta)} \\
&= \sum_{x} \frac{J_z(x)^2}{K_z(x)}  -2\mathbb{E} (y|x) J_z(x) + 
 \mathbb{E}(y^2|x) K_z(x).
\end{align*}
The Hessian matrix of $J_z^2/K_z$ is
\begin{equation*}
\nabla_z^2 \left(\frac{J_z^2}{K_z}\right)
 = \frac{1}{K_z}(h_D -h_z^\eta D)(h_D-h_z^\eta D)^T
\end{equation*}
where $h_D=(h_1 \sD_1,h_2 \sD_2,\dots,h_p \sD_p )^T$
and $D = (\sD_1, \sD_2, \dots,\sD_p)^T$. Thus $J_z^2/K_z$ is
convex. $-2\mathbb{E} (y|x) J_z(x) +  \mathbb{E}(y^2|x) K_z(x)$ 
is an affine function of $z$ and is therefore convex. 
Therefore the first term of $v_k(z)$ is convex, which completes the proof.  
\end{proof}

\vspace{-.5cm}%
Thus, the optimization problem can be cast as the following
variational form of a DC-programming problem
\citep{TaoAn1997,TaoAn1998,SriperumbudurLanckriet2009}:
 \begin{align}
\label{eq:dc}
\min_{z \in \Delta, \gamma \in \Rset} & \ \gamma\\\nonumber
\text{s.t.} \ u_k(z) - v_k(z) & \leq
\gamma, & \forall k \in [p]\\\nonumber
-z_k & \leq 0  & \forall k \in [p]\\\nonumber
\textstyle \sum_{k = 1}^p z_k - 1 & = 0.
\end{align}
Let $(z_t)_t$ be the sequence defined by repeatedly solving
the following convex optimization problem:
\begin{align}
\label{eq:dca}
z_{t + 1} \in \argmin_{z, \gamma \in \Rset} & \ \gamma\\\nonumber
\text{s.t.} \ u_k(z) - v_k(z_t) - (z - z_t)^T \nabla v_k(z_t) & \leq
\gamma, \forall k \in [p]\\\nonumber
-z_k & \leq 0, \forall k \in [p]\\\nonumber
\textstyle \sum_{k = 1}^p z_k - 1 & = 0,
\end{align}
where $z_0 \in \Delta$ is an arbitrary starting value. Then, $(z_t)_t$
is guaranteed to converge to a stationary point of Problem~\ref{eq:opt}
\citep{YuilleRangarajan2003,SriperumbudurLanckriet2009}. 
Note that Problem~\ref{eq:dca} is a relatively simple optimization problem:
$u_k(z)$ is a weighted sum of rational functions of $z$ and all other terms 
appearing in the constraints are
affine functions of $z$. 

Our optimization problem (Problem~\ref{eq:opt}) seeks a parameter $z$
verifying $\cL(\sD_k, h_z^\eta) - \cL(\sD_z, h_z^\eta) \leq \gamma$,
for all $k \in [p]$ for an arbitrarily small value of $\gamma$.  Since
$\cL(\sD_z, h_z^\eta) = \sum_{k = 1}^p z_k \cL(\sD_k, h_z^\eta)$ is a
weighted average of the expected losses $\cL(\sD_k, h_z^\eta)$,
$k \in [p]$, the solution $\gamma$ cannot be negative.  Furthermore,
by Lemma~\ref{lemma:brouwer}, a parameter $z$ verifying that
inequality exists for any $\gamma > 0$. Thus, the global solution
$\gamma$ of Problem~\ref{eq:opt} must be close to zero. This provides
us with a simple criterion for testing the global optimality of the
solution $z$ we obtain using a DC-programming algorithm with a
starting parameter $z_0$.

\vspace{-.5cm}%

\section{Experiments}
\label{sec:eval}

This section reports the results of our experiments with our
DC-programming algorithm for finding a robust domain generalization
solution when using the squared loss. We first evaluate our algorithm
using an artificial dataset assuming known densities where we may 
compare our result to the global solution.  
Next, we evaluate our DC-programming
solution applied to a real-world sentiment analysis dataset~\cite{blitzer_acl07}.


\ignore{All code is implemented using standard
software packages for our
DC-programming solution, and will be released upon acceptance.}

\subsection{Artificial dataset}

\begin{table*}
\begin{center}
\scriptsize{
\begin{tabular}{l ccccccccccc}
\toprule
& \multicolumn{9}{c}{Test Data}\\
\cline{2-12}
& \texttt{K}& \texttt{D} & \texttt{B}	& \texttt{E}	&\texttt{KD}	&\texttt{BE}	& \texttt{DBE} & \texttt{KBE} & \texttt{KDB} &	\texttt{KDB} &	\texttt{KDBE} \\
\midrule
\texttt{K}      &       {1.46$\pm$0.08} &       {2.20$\pm$0.14} &       {2.29$\pm$0.13} &       {1.69$\pm$0.12} &       {1.83$\pm$0.08} &       {1.99$\pm$0.10} &       {2.06$\pm$0.07}      &       {1.81$\pm$0.07} &       {1.78$\pm$0.07} &       {1.98$\pm$0.06} &       {1.91$\pm$0.06}\\
\texttt{D}      &       {2.12$\pm$0.08} &       {1.78$\pm$0.08} &       {2.12$\pm$0.08} &       {2.10$\pm$0.07} &       {1.95$\pm$0.07} &       {2.11$\pm$0.07} &       {2.00$\pm$0.06}      &       {2.11$\pm$0.06} &       {2.00$\pm$0.06} &       {2.01$\pm$0.06} &       {2.03$\pm$0.06}\\
\texttt{B}      &       {2.18$\pm$0.11} &       {2.01$\pm$0.09} &       {1.73$\pm$0.12} &       {2.24$\pm$0.07} &       {2.10$\pm$0.09} &       {1.99$\pm$0.08} &       {1.99$\pm$0.05}      &       {2.05$\pm$0.06} &       {2.14$\pm$0.06} &       {1.98$\pm$0.06} &       {2.04$\pm$0.05}\\
\texttt{E}      &       {1.69$\pm$0.09} &       {2.31$\pm$0.12} &       {2.40$\pm$0.11} &       {1.50$\pm$0.06} &       {2.00$\pm$0.09} &       {1.95$\pm$0.07} &       {2.07$\pm$0.06}      &       {1.86$\pm$0.04} &       {1.84$\pm$0.06} &       {2.14$\pm$0.06} &       {1.98$\pm$0.05}\\
\texttt{unif}   &       {1.62$\pm$0.05} &       {1.84$\pm$0.09} &       {1.86$\pm$0.09} &       {1.62$\pm$0.07} &       {1.73$\pm$0.06} &       {1.74$\pm$0.07} &       {1.77$\pm$0.05}      &       {1.70$\pm$0.05} &       {1.69$\pm$0.04} &       {1.77$\pm$0.04} &       {1.74$\pm$0.04}\\
\texttt{KMM}    &       1.63$\pm$0.15   &       2.07$\pm$0.12   &       1.93$\pm$0.17   &       1.69$\pm$0.12   &       1.83$\pm$0.07   &       1.82$\pm$0.07   &       1.89$\pm$0.07&       1.75$\pm$0.07   &       1.78$\pm$0.06   &       1.86$\pm$0.09   &       1.82$\pm$0.06\\
\texttt{DW}     &       {\bf1.45$\pm$0.08}      &       {\bf1.78$\pm$0.08}      &       {\bf1.72$\pm$0.12}      &       {\bf1.49$\pm$0.06}      &       {\bf1.62$\pm$0.07}      &   {\bf1.61$\pm$0.08}       &       {\bf1.66$\pm$0.05}      &       {\bf1.56$\pm$0.04}      &       {\bf1.58$\pm$0.05}      &       {\bf1.65$\pm$0.04}      &       {\bf1.61$\pm$0.04}\\
\bottomrule
\end{tabular}
}
\end{center}
\caption{MSE on the sentiment analysis dataset of source-only
  baselines for each domain, \texttt{K, D, B, E}, 
  the uniform weighted predictor
  \texttt{unif}, \texttt{KMM}, and the distribution weighted method
  \texttt{DW} based on the learned $z$. \texttt{DW} outperforms all
  competing baselines. For all domain combinations see supplementary material.}
\label{table:sa}	
\end{table*}

We first evaluated our algorithm on a synthetic dataset. Here we
considered the following two-dimensional setup, proposed for multiple source domain study by~\citet{MansourMohriRostamizadeh2009}.  Let $g_1$, $g_2$, $g_3$, $g_4$
denote the Gaussian distributions with means $(1, 1)$, $(-1, 1)$,
$(-1, -1)$, and $(1, -1)$ and unit variance respectively.  Each domain
was generated as a uniform mixture of Gaussians: $\sD_1$ from
$\{g_1, g_2, g_3\}$ and $\sD_2$ from $\{g_2, g_3, g_4\}$.  The
labeling function is $f(x_1, x_2) = x_1^2 + x_2^2$.  We trained linear
regressors for each domain to produce base hypotheses $h_1$ and
$h_2$. Finally, as the true distribution is known for this artificial
example, we directly use the Gaussian mixture
density function to generate our $\sD_k$s.

With this data source, we used our DC-programming solution to find the
optimal mixing weights $z$. Figure~\ref{fig:synthetic_loss} shows 
the global loss vs number of iterations with the uniform
initialization $z_0 = [1/2, 1/2]$. Here, the overall objective
approaches $0.0$, the known global minimum. To verify the robustness
of the solution, we have experimented with various initial conditions
and found that the solution converges to the global solution in each
case.

\begin{figure}[t]
\centering
\includegraphics[height=.4\linewidth]{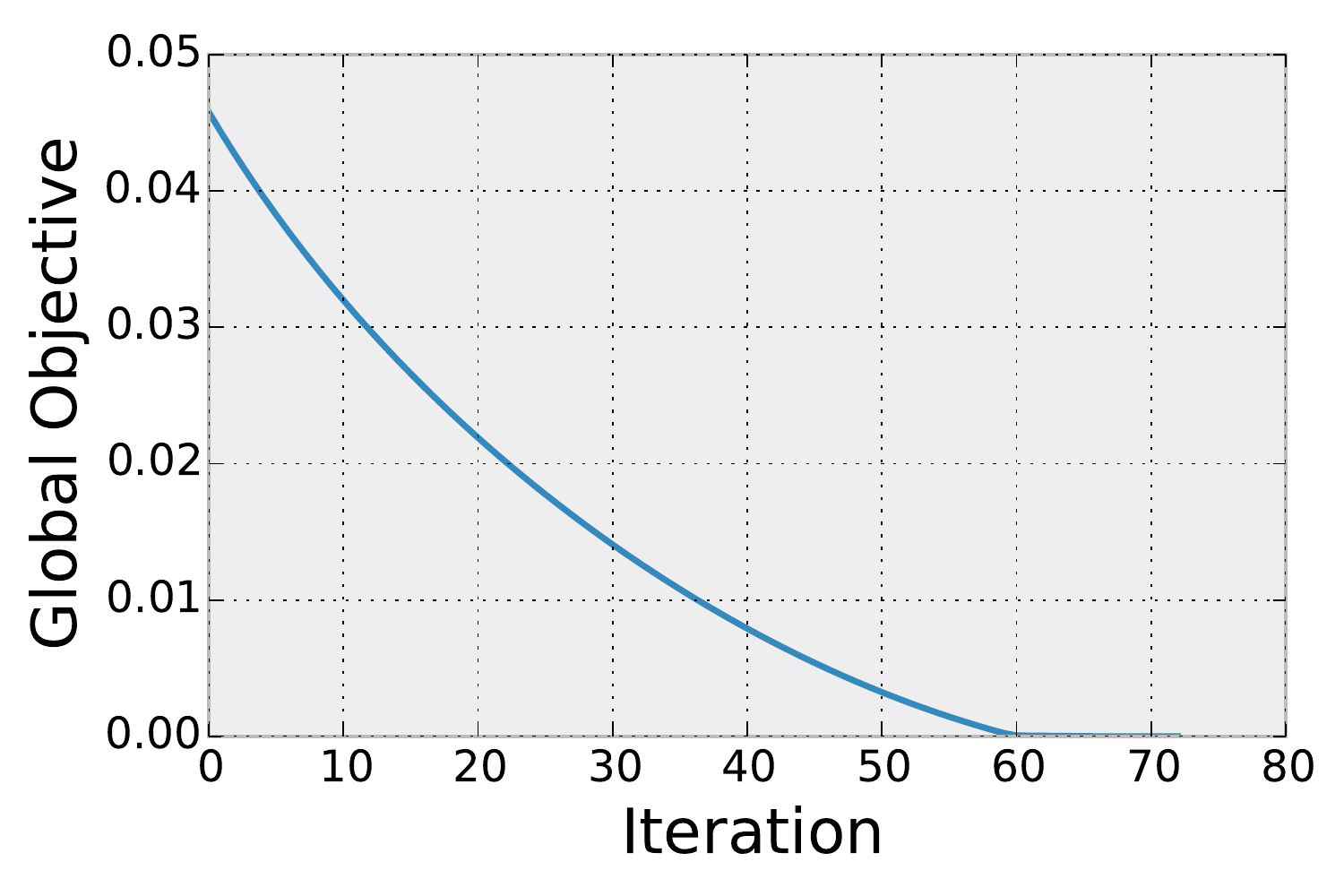}
\caption{Synthetic global loss versus iteration. Our solution converges to the global optimum of zero.}
\vskip -.15in
\label{fig:synthetic_loss}
\end{figure}

\subsection{Sentiment analysis task}

With our optimization solution verified on the synthetic example, with
known densities, we now seek to evaluate our performance on a
real-world task with estimated densities.  There are few readily
available datasets for studying multiple source adaptation and even
fewer for adaptation with regression tasks.  We use the sentiment
analysis dataset proposed by~\citet{blitzer_acl07} and used for
multiple-source adaptation by~\citet{MansourMohriRostamizadeh2008,
  MansourMohriRostamizadeh2009}.  This dataset consists of product
review text and rating labels taken from four domains: \texttt{books}
(B), \texttt{dvd} (D), \texttt{electronics} (E), and
$\texttt{kitchen}$ (K), with $2000$ samples for each domain.  We
create 10 random splits of our data into a training set ($1600$
points) and test set ($400$ points) per domain.


\textbf{Predictors:} We defined a vocabulary of $2\mathord,500$ words
that occur at least twice at the intersection of the four
domains. These words were used to define feature vectors, where every
sample point is encoded by the number of occurrences of each word. We
trained our base hypotheses using support vector regression
(SVR)\footnote{We use the libsvm package implementation
  \texttt{http://www.csie.ntu.edu.tw/÷cjlin/libsvm/}} with same
hyper-parameters as in \citep{MansourMohriRostamizadeh2008,
  MansourMohriRostamizadeh2009}.

 \begin{figure*}[t]
\vskip -.15in
  	\centering
 	\subfloat{
  	\includegraphics[width=.4\linewidth]{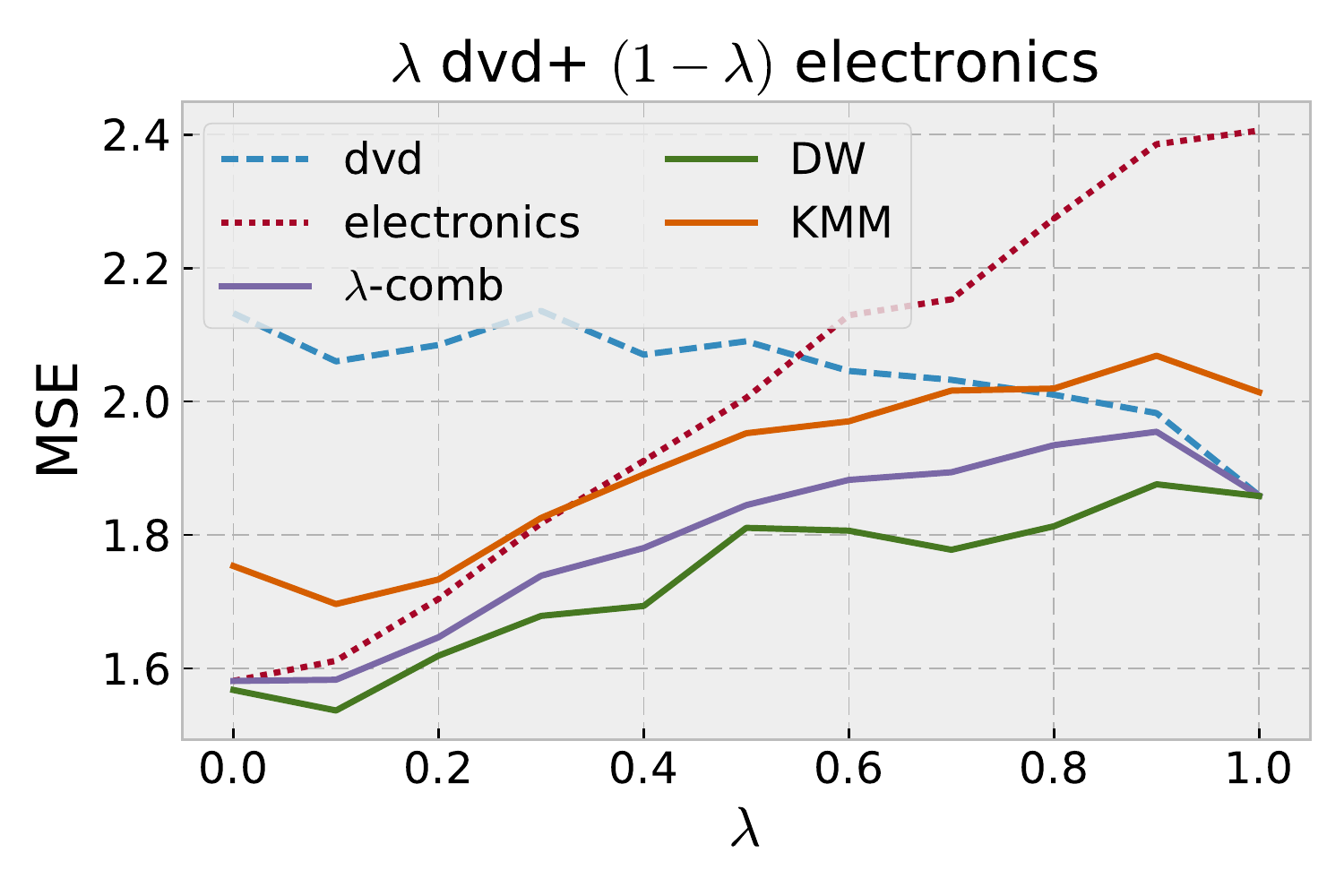}
  	\label{fig:2domain_DE}
 	}
 	\hspace{2cm}%
 	\subfloat{
  	\includegraphics[width=.4\linewidth]{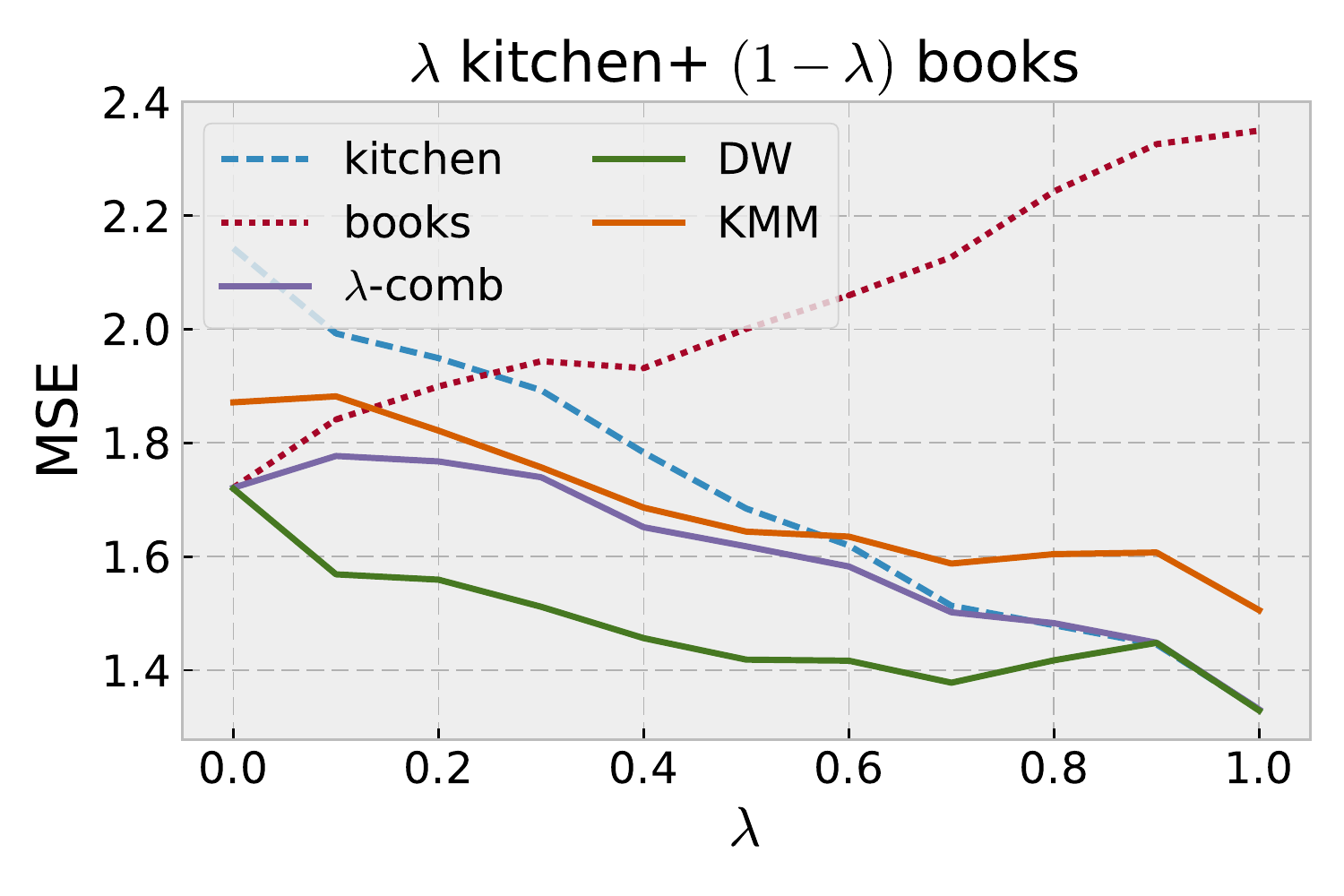}
  	\label{fig:2domain_KB}
 	}
 	\caption{\textbf{MSE sentiment analysis} under mixture of two domains: (\emph{left}) dvd and electronics or (\emph{right}) kitchen and books.}
\vskip -.15in
\end{figure*} 
 
\textbf{Density estimation:} In practice, the probability
distributions $\sD_k$ are not readily available to the
learner. However, Theorem~\ref{th:estimate} extends the learning
guarantees of our solution to the case where an estimate $\h \sD_k$ is
used in lieu of the ideal distribution $\sD_k$, for each $k \in
[p]$. Thus, here, we briefly discuss the problem of deriving an
estimate $\h \sD_k$ for each $k \in [p]$. 

We used the same vocabulary defined for feature extraction to train a
bigram statistical language model for each domain, using the OpenGrm
library \citep{roark2012opengrm}.
Next, we randomly draw a sample set $S_k$ of $10\mathord,000$ sentences
from each bigram language model. We define $\h \sD_k$ to be the 
empirical distribution of $S_k$, which is a very close estimate of the
marginal distribution of the language model, thus it is also a good estimate 
of $\sD_k$. We approximate the label of a randomly generated sample
$x_i$ by taking the average of the $h_k$s: 
$y_i = \sum_{\{k\colon x_i \in S_k\}} h_k(x_i) / \vert \{k\colon x_i \in S_k\}\vert$.
 These randomly
drawn samples were used to find the fixed-point $z$.  

Note that we only use estimates of the marginal distributions
(language models) to find $z$ and do not use any labels.  We use the
original product review text and rating labels for testing. Their
densities $\h \sD_k$ were estimated by the bigram language models
directly, therefore a close estimate of $\sD_k$.

\textbf{Baselines:} We compare against each source hypothesis, $h_k$,
as well as the uniform combination of all hypotheses (\texttt{unif}),
$\sum_{k=1}^p h_k/p$.  We also compute a privileged baseline using the
known $\lambda$ mixing parameter, \texttt{$\lambda$-comb}:
$\sum_{k=1}^p \lambda_k h_k$.
 %
 $\lambda$-\texttt{comb} is of course not accessible
 in practice since the target mixture $\lambda$ is not
 known to the user.

 Note that our method (\texttt{DW}) produces a single predictor for
 \emph{any} target distribution which is a mixture of source
 domains. For completeness, we compare against a previously proposed
 domain adaptation
 algorithm~\citep{HuangSmolaGrettonBorgwardtScholkopf2006} known as
 KMM.  However, it is important to note that the KMM model requires
 access to the unlabeled target data during adaptation and learns a
 new predictor per target domain.  Thus KMM operates in a favorable
 learning setting when compared to our solution.
 
 %
 Given source data and unlabeled target data, KMM reweights samples
 from the source domain so that the mean input feature vector over the
 source domain is close to the mean input feature vector over target
 domain in reproducing kernel Hilbert space (RKHS).  Then, it uses
 these sample weights on the source data to train a predictor using
 SVR.  In our KMM experiments, we used the same kernel function as the
 one use for training the base hypotheses.  We used the
 hyper-paramters $B = 1000$ and
 $\epsilon = \frac{\sqrt{n}}{\sqrt{n} - 1}$, where $n$ is number of
 training samples, as recommended by the authors
 \citep{HuangSmolaGrettonBorgwardtScholkopf2006}. We denote by
 \texttt{KMM} the specifically trained predictor for each test data.


\textbf{Prediction:} 
We first considered the scenario where the target is a mixture of two
source domains.  For $\lambda$ varying from $0, 0.1, 0.2, \ldots, 1$,
the test set consists of $400 \cdot \lambda$ points from the first
domain and $400 \cdot (1 - \lambda)$ points from second domain. We
used $1600$ training points for KMM ($800$ per domain), which
coincides with the sample size used to learn each base hypothesis, and
used $50\%$ of the test dataset for KMM adaptation. We report
experimental results in Figures~\ref{fig:2domain_DE}
and~\ref{fig:2domain_KB}.  They show that our distribution weighted
predictor \texttt{DW} outperforms all baseline predictors despite the
privileged learning scenarios of \texttt{$\lambda$-comb} and
\texttt{KMM}.

Next, we compared the performance of \texttt{DW} with the baseline
predictors on various target mixtures, including each domain
individually, the combination of any two and three domains, and all
four domains. More mixture combinations are available in the
supplementary material.  We used again $1600$ training points ($400$
per domain) and $50\%$ testing data for KMM algorithm.
Table~\ref{table:sa} reports the mean and standard deviations of MSE
over $10$ repetitions.  Each column corresponds to a different target
test data source.  Our distribution weighted method \texttt{DW}
outperforms all baseline predictors across all test domains. Observe
that, even when the target is a single source domain, our method
successfully outperforms the predictor which is trained and tested on
the same domain.  Notice that this can happen even when each $h_k$ is
the best predictor in some hypothesis set $H_k$ for its own domain,
since our distribution weighted predictor \texttt{DW} belongs to a
more complex family of hypotheses.

Overall, these results show that our single \texttt{DW} predictor
has strong performance
across various test domains, confirming the robustness suggested
by our theory.

\ignore{
\subsection{Human joint detection task}
\input{5_3_exp_coco_mpii}
}

\section{Conclusion}
\label{sec:conclusion}

We presented an algorithm for multiple-source domain adaptation in the
common scenario where the squared loss is used. Our algorithm computes
a distribution-weighted combination of predictors for each source
domain which we showed to admit very favorable guarantees.
We further demonstrated the effectiveness of our algorithm empirically in experiments with both artificial data sets and a sentiment analysis task. 

\bibliographystyle{abbrvnat} 
\bibliography{madap}

\begin{thebibliography}{39}
\providecommand{\natexlab}[1]{#1}
\providecommand{\url}[1]{\texttt{#1}}
\expandafter\ifx\csname urlstyle\endcsname\relax
  \providecommand{\doi}[1]{doi: #1}\else
  \providecommand{\doi}{doi: \begingroup \urlstyle{rm}\Url}\fi

\bibitem[Arndt(2004)]{arndt}
C.~Arndt.
\newblock \emph{Information Measures: Information and its Description in
  Science and Engineering}.
\newblock Signals and Communication Technology. Springer Verlag, 2004.

\bibitem[Blitzer et~al.(2007)Blitzer, Dredze, and Pereira]{blitzer_acl07}
J.~Blitzer, M.~Dredze, and F.~Pereira.
\newblock Biographies, bollywood, boom-boxes and blenders: Domain adaptation
  for sentiment classification.
\newblock In \emph{Association for Computational Linguistics (ACL)}, 2007.

\bibitem[Chattopadhyay et~al.(2012)Chattopadhyay, Sun, Fan, Davidson,
  Panchanathan, and Ye]{chattopadhyay2012multisource}
R.~Chattopadhyay, Q.~Sun, W.~Fan, I.~Davidson, S.~Panchanathan, and J.~Ye.
\newblock Multisource domain adaptation and its application to early detection
  of fatigue.
\newblock \emph{ACM Transactions on Knowledge Discovery from Data (TKDD)},
  6\penalty0 (4):\penalty0 18, 2012.

\bibitem[Chen and Deng(2008)]{ChenDeng2008}
X.~Chen and X.~Deng.
\newblock Matching algorithmic bounds for finding a brouwer fixed point.
\newblock \emph{J. {ACM}}, 55\penalty0 (3), 2008.

\bibitem[Dredze et~al.(2008)Dredze, Crammer, and Pereira]{dredze_nips08}
M.~Dredze, K.~Crammer, and F.~Pereira.
\newblock Confidence-weighted linear classification.
\newblock In \emph{International Conference on Machine Learning (ICML)}, 2008.

\bibitem[Duan et~al.(2012{\natexlab{a}})Duan, Xu, and
  Chang]{duan2012exploiting}
L.~Duan, D.~Xu, and S.-F. Chang.
\newblock Exploiting web images for event recognition in consumer videos: A
  multiple source domain adaptation approach.
\newblock In \emph{Computer Vision and Pattern Recognition (CVPR), 2012 IEEE
  Conference on}, pages 1338--1345. IEEE, 2012{\natexlab{a}}.

\bibitem[Duan et~al.(2012{\natexlab{b}})Duan, Xu, and Tsang]{duan2012domain}
L.~Duan, D.~Xu, and I.~W.-H. Tsang.
\newblock Domain adaptation from multiple sources: A domain-dependent
  regularization approach.
\newblock \emph{IEEE Transactions on Neural Networks and Learning Systems},
  23\penalty0 (3):\penalty0 504--518, 2012{\natexlab{b}}.

\bibitem[Eaves(1972)]{Eaves1972}
B.~C. Eaves.
\newblock Matching algorithmic bounds for finding a brouwer fixed point.
\newblock \emph{Math. Progr.}, 3:\penalty0 1--22, 1972.

\bibitem[Ganin and Lempitsky(2015)]{ganin_icml15}
Y.~Ganin and V.~Lempitsky.
\newblock Unsupervised domain adaptation by backpropagation.
\newblock In \emph{International Conference in Machine Learning (ICML)}, 2015.

\bibitem[Girshick et~al.(2014)Girshick, Donahue, Darrell, and Malik]{rcnn}
R.~Girshick, J.~Donahue, T.~Darrell, and J.~Malik.
\newblock Rich feature hierarchies for accurate object detection and semantic
  segmentation.
\newblock In \emph{In Proc. CVPR}, 2014.

\bibitem[Gong et~al.(2012)Gong, Shi, Sha, and Grauman]{gong_cvpr12}
B.~Gong, Y.~Shi, F.~Sha, and K.~Grauman.
\newblock Geodesic flow kernel for unsupervised domain adaptation.
\newblock In \emph{Proc. CVPR}, 2012.

\bibitem[Gong et~al.(2013{\natexlab{a}})Gong, Grauman, and Sha]{gong_iccv13}
B.~Gong, K.~Grauman, and F.~Sha.
\newblock Connecting the dots with landmarks: Discriminatively learning
  domain-invariant features for unsupervised domain adaptation.
\newblock In \emph{ICCV}, 2013{\natexlab{a}}.

\bibitem[Gong et~al.(2013{\natexlab{b}})Gong, Grauman, and Sha]{gong_nips13}
B.~Gong, K.~Grauman, and F.~Sha.
\newblock Reshaping visual datasets for domain adaptation.
\newblock In \emph{NIPS}, 2013{\natexlab{b}}.

\bibitem[Hirsch et~al.(1989)Hirsch, Papadimitriou, and
  Vavasis]{HirschPapadimitriouVavasis1989}
M.~D. Hirsch, C.~H. Papadimitriou, and S.~A. Vavasis.
\newblock Exponential lower bounds for finding brouwer fix points.
\newblock \emph{J. Complexity}, 5\penalty0 (4):\penalty0 379--416, 1989.

\bibitem[Hoffman et~al.(2012)Hoffman, Kulis, Darrell, and
  Saenko]{hoffman_eccv12}
J.~Hoffman, B.~Kulis, T.~Darrell, and K.~Saenko.
\newblock Discovering latent domains for multisource domain adaptation.
\newblock In \emph{European Conference on Computer Vision (ECCV)}, 2012.

\bibitem[Hoffman et~al.(2013)Hoffman, Rodner, Donahue, Saenko, and
  Darrell]{hoffman_iclr13}
J.~Hoffman, E.~Rodner, J.~Donahue, K.~Saenko, and T.~Darrell.
\newblock Efficient learning of domain-invariant image representations.
\newblock In \emph{International Conference on Learning Representations}, 2013.

\bibitem[Horst and Thoai(1999)]{HorstThoai1999}
R.~Horst and N.~V. Thoai.
\newblock {DC} programming: overview.
\newblock \emph{Journal of Optimization Theory and Applications}, 103\penalty0
  (1):\penalty0 1--43, 1999.

\bibitem[Huang et~al.(2006)Huang, Smola, Gretton, Borgwardt, and
  Sch{\"o}lkopf]{HuangSmolaGrettonBorgwardtScholkopf2006}
J.~Huang, A.~J. Smola, A.~Gretton, K.~M. Borgwardt, and B.~Sch{\"o}lkopf.
\newblock Correcting sample selection bias by unlabeled data.
\newblock In \emph{Advances in Neural Information Processing Systems (NIPS)},
  volume~19, pages 601--608, 2006.

\bibitem[Kuhn(1968)]{Kuhn1968}
H.~Kuhn.
\newblock Simplicial approximations of fixed points.
\newblock \emph{Proceedings of the National Academy of Sciences}, 61:\penalty0
  1238--1242, 1968.

\bibitem[Liao(2013)]{liao_icassp13}
H.~Liao.
\newblock Speaker adaptation of context dependent deep neural networks.
\newblock In \emph{ICASSP}, 2013.

\bibitem[Long et~al.(2015)Long, Cao, Wang, and Jordan]{long_icml15}
M.~Long, Y.~Cao, J.~Wang, and M.~I. Jordan.
\newblock Learning transferable features with deep adaptation networks.
\newblock In \emph{International Conference in Machine Learning (ICML)}, 2015.

\bibitem[Mansour et~al.(2008)Mansour, Mohri, and
  Rostamizadeh]{MansourMohriRostamizadeh2008}
Y.~Mansour, M.~Mohri, and A.~Rostamizadeh.
\newblock Domain adaptation with multiple sources.
\newblock In \emph{NIPS}, 2008.

\bibitem[Mansour et~al.(2009)Mansour, Mohri, and
  Rostamizadeh]{MansourMohriRostamizadeh2009}
Y.~Mansour, M.~Mohri, and A.~Rostamizadeh.
\newblock Multiple source adaptation and the {R}{\'e}nyi divergence.
\newblock In \emph{UAI}, pages 367--374, 2009.

\bibitem[Merrill(1972)]{Merrill1972}
O.~H. Merrill.
\newblock \emph{Applications and Extensions of an Algorithm That Computes Fixed
  Points of Certain Upper Semi-continuous Point to Set Mappings}.
\newblock PhD thesis, Dept. of Industrial Engineering, University of Michigan,
  1972.

\bibitem[Muandet et~al.(2013)Muandet, Balduzzi, and
  Sch{\"{o}}lkopf]{MuandetBalduzziScholkopf2013}
K.~Muandet, D.~Balduzzi, and B.~Sch{\"{o}}lkopf.
\newblock Domain generalization via invariant feature representation.
\newblock In \emph{Proceedings of {ICML}}, pages 10--18, 2013.

\bibitem[Pan and Yang(2010)]{pan_tkda2010}
S.~J. Pan and Q.~Yang.
\newblock A survey on transfer learning.
\newblock In \emph{IEEE Transactions on Knowledge and Data Engineering}, 2010.

\bibitem[R\'enyi(1961)]{Renyi1961}
A.~R\'enyi.
\newblock On measures of entropy and information.
\newblock In \emph{Proceedings of the Fourth Berkeley Symposium on Mathematical
  Statistics and Probability}, volume~1, pages 547--561, 1961.

\bibitem[Roark et~al.(2012)Roark, Sproat, Allauzen, Riley, Sorensen, and
  Tai]{roark2012opengrm}
B.~Roark, R.~Sproat, C.~Allauzen, M.~Riley, J.~Sorensen, and T.~Tai.
\newblock The opengrm open-source finite-state grammar software libraries.
\newblock In \emph{Proceedings of the ACL 2012 System Demonstrations}, pages
  61--66. Association for Computational Linguistics, 2012.

\bibitem[Saenko et~al.(2010)Saenko, Kulis, Fritz, and Darrell]{saenko_eccv10}
K.~Saenko, B.~Kulis, M.~Fritz, and T.~Darrell.
\newblock Adapting visual category models to new domains.
\newblock In \emph{Proc. ECCV}, 2010.

\bibitem[Scarf(1967)]{Scarf1967}
H.~Scarf.
\newblock The approximation of fixed points of a continuous mapping.
\newblock \emph{SIAM J. Appl. Math}, 15\penalty0 (5), 1967.

\bibitem[Schweikert et~al.(2009)Schweikert, R{\"a}tsch, Widmer, and
  Sch{\"o}lkopf]{schweikert2009empirical}
G.~Schweikert, G.~R{\"a}tsch, C.~Widmer, and B.~Sch{\"o}lkopf.
\newblock An empirical analysis of domain adaptation algorithms for genomic
  sequence analysis.
\newblock In \emph{Advances in Neural Information Processing Systems}, pages
  1433--1440, 2009.

\bibitem[Sriperumbudur and Lanckriet(2012)]{SriperumbudurLanckriet2009}
B.~K. Sriperumbudur and G.~R.~G. Lanckriet.
\newblock A proof of convergence of the concave-convex procedure using
  {Z}angwill's theory.
\newblock \emph{Neural Computation}, 24\penalty0 (6):\penalty0 1391--1407,
  2012.

\bibitem[Tao and An(1997)]{TaoAn1997}
P.~D. Tao and L.~T.~H. An.
\newblock Convex analysis approach to {DC} programming: theory, algorithms and
  applications.
\newblock \emph{Acta Mathematica Vietnamica}, 22\penalty0 (1):\penalty0
  289--355, 1997.

\bibitem[Tao and An(1998)]{TaoAn1998}
P.~D. Tao and L.~T.~H. An.
\newblock A {DC} optimization algorithm for solving the trust-region
  subproblem.
\newblock \emph{SIAM Journal on Optimization}, 8\penalty0 (2):\penalty0
  476--505, 1998.

\bibitem[Torralba and Efros(2011)]{efros_cvpr11}
A.~Torralba and A.~Efros.
\newblock Unbiased look at dataset bias.
\newblock In \emph{CVPR}, 2011.

\bibitem[Tzeng et~al.(2015)Tzeng, Hoffman, Darrell, and Saenko]{tzeng_iccv15}
E.~Tzeng, J.~Hoffman, T.~Darrell, and K.~Saenko.
\newblock Simultaneous deep transfer across domains and tasks.
\newblock In \emph{International Conference in Computer Vision (ICCV)}, 2015.

\bibitem[Xu et~al.(2014)Xu, Li, Niu, and Xu]{xu_eccv14}
Z.~Xu, W.~Li, L.~Niu, and D.~Xu.
\newblock Exploiting low-rank structure from latent domains for domain
  generalization.
\newblock In \emph{European Conference in Computer Vision (ECCV)}, 2014.

\bibitem[Yang et~al.(2007)Yang, Yan, and Hauptmann]{yang_acmm07}
J.~Yang, R.~Yan, and A.~G. Hauptmann.
\newblock Cross-domain video concept detection using adaptive svms.
\newblock \emph{ACM Multimedia}, 2007.

\bibitem[Yuille and Rangarajan(2003)]{YuilleRangarajan2003}
A.~L. Yuille and A.~Rangarajan.
\newblock The concave-convex procedure.
\newblock \emph{Neural Computation}, 15\penalty0 (4):\penalty0 915--936, 2003.

\end{thebibliography}

\clearpage
\appendix

In this appendix, we give detailed proofs of several of theorems for
multiple-source adaptation for regression (Section~\ref{app:pf}. We
present an analysis of the stationary points for our min-max
optimization problem in Section~\ref{app:stationary}.  Finally, in
Section~\ref{app:more_exp}, we report the results of several
additional experiments demonstrating the benefits of our robust
solutions.

\section{Proofs}
\label{app:pf}

\begin{reptheorem}{th:mixture}
For any $\delta > 0$, there exists $\eta > 0$ and $z\in \Delta$, such
that $\cL(\sD_\lambda,h_z^\eta)\leq \e +\delta$ \ for any mixture
parameter $\lambda \in \Delta$.
\end{reptheorem}

\begin{proof}
  We first upper bound, for an arbitrary $z \in \Delta$, the expected
loss of $h_z^\eta$ with respect to the mixture distribution $\sD_z$ defined
using the same $z$, that is $\cL(\sD_z, h_z^\eta) = \sum_{k = 1}^p z_k
\cL(\sD_k, h_z^\eta)$.
By definition of $h_z^\eta$ and $\sD_z$, we can write
\begin{align*}
& \cL(\sD_z, h_z^\eta)\\
& = \sum_{(x, y) } \sD_z(x,y) L(h_z^\eta(x), y)\\
& = \sum_{(x, y)} \sD_z(x, y) L \mspace{-2mu} \left(\mspace{-1mu}
  \sum_{k = 1}^p \frac{z_k
    \msD_k(x) + \eta  \frac{\msU(x)}{p}}{\msD_z(x) + \eta
    \msU(x)} h_k(x), y \right) \mspace{-4mu}.
\end{align*}
By convexity of $L$, this implies that
\begin{align*}
\lefteqn{\cL(\sD_z, h_z^\eta)}\\
& \leq \sum_{(x, y)} \sD_z(x, y) \sum_{k=1}^p \frac{z_k
    \msD_k(x) + \eta  \frac{\msU(x)}{p}}{\msD_z(x) + \eta
    \msU(x)} L\big(h_k(x), y \big)\\
& \leq \sum_{(x, y)} \sD_z(y | x) \msD_z(x) \sum_{k=1}^p \frac{z_k
    \msD_k(x) + \eta  \frac{\msU(x)}{p}}{\msD_z(x) + \eta
    \msU(x)} L\big(h_k(x), y \big)\\
& \leq \sum_{(x, y)} \sD_z(y | x) \sum_{k=1}^p \bigg( z_k
    \msD_k(x) + \eta  \frac{\msU(x)}{p} \bigg) L\big(h_k(x), y \big).
\end{align*}
Next, observe that
$\sD_{z}(y|x)=\sum_{k=1}^{p}\frac{z_{k}\msD_{k}(x)}{\msD_z(x)}
\sD_{k}(y|x)=\sD_k(y|x)$ for any $k \in [1, p]$ 
since by assumption $\sD_{k}(y | x)$ does not
depend on $k$. Thus,
\begin{align*}
\lefteqn{\cL(\sD_z, h_z^\eta)}\\
& \leq \sum_{(x, y)} \sD_z(y | x) \sum_{k=1}^p \bigg( z_k
    \msD_k(x) + \eta  \frac{\msU(x)}{p} \bigg) L\big(h_k(x), y \big)\\
    &  =  \sum_{(x, y)} \mspace{-2mu}\sum_{k=1}^p \bigg(\mspace{-2mu} z_k
    \sD_k(y | x) \msD_k(x) + \eta  \sD_k(y | x) \frac{\msU(x)}{p} \mspace{-2mu}\bigg)
  \mspace{-2mu} L\big(h_k(x), y \big)\\
  & = \sum_{k = 1}^p z_k \cL(\sD_k, h_k) + \frac{\eta}{p} \sum_{k=1}^p \sum_{(x,
  y)} \sD_k(y | x) \msU(x)
  L\big(h_k(x), y \big)\\
  &  \leq  \sum_{k = 1}^p z_k \cL(\sD_k, h_k) + \eta M\\
  & \leq  \sum_{k = 1}^p z_k \e + \eta M = \e + \eta M.
\end{align*}
Now, choose $z \in \Delta$ as in the statement of
Lemma~\ref{lemma:brouwer}. Then,
the following holds for any mixture distribution $\sD_\lambda$:
\begin{align*}
\cL(\sD_\lambda,h_z^\eta)
&= \sum_{k=1}^p \lambda_k \cL(\sD_k,h_z^\eta) \\
& \leq \sum_{k=1}^p \lambda_k \big( \cL(\sD_z,h_z^\eta) +\eta' \big)\\
& = \cL(\sD_z,h_z^\eta) +\eta'
 \leq \e + \eta M + \eta'.
\end{align*}
Setting $\eta = \frac{\delta}{2M}$ and $\eta' = \frac{\delta}{2}$
concludes the proof.
\end{proof}

\begin{reptheorem}{th:arbitrary}
Let $\sD_T$ be an arbitrary target distribution.
For any $\delta > 0$, there exists
$\eta > 0$ and $z \in \Delta$, such that
the following inequality holds for any $\alpha > 1$:
\begin{equation*}
\cL(\sD_T, h_z^\eta) 
\leq \Big[(\e + \delta) \, \sfd_\alpha(\sD_T \parallel \cD)
\Big]^{\frac{\alpha - 1}{\alpha}} M^{\frac{1}{\alpha }}.
\end{equation*}
\end{reptheorem}
\begin{proof}
  For any hypothesis $h\colon \cX \to \cY $ and any
  distribution $\sD$, by H\"older's inequality, the following holds:
\begin{align*}
\cL&(\sD_T, h)
 = \sum_{(x, y) \in \cX \times \cY} \sD_T(x, y) L(h(x), y)\\
& = \sum_{(x, y) \in \cX \times \cY} \left [\frac{\sD_T(x, y)}{\sD(x,
  y)^{\frac{\alpha - 1}{\alpha}}}\right] \left[ \sD(x,
  y)^{\frac{\alpha - 1}{\alpha}} L(h(x), y) \right]\\
& \leq \left [\sum_{(x, y)} \frac{\sD_T(x, y)^\alpha}{\sD(x,
  y)^{\alpha - 1}}\right]^{\frac{1}{\alpha}} \left[ \sum_{(x, y)} \sD(x,
  y) L(h(x), y)^{\frac{\alpha}{\alpha - 1}} \right]^{\mspace{-8mu}\frac{\alpha -
  1}{\alpha}} \mspace{-25mu}.
\end{align*}
Thus, by definition of $\sfd_\alpha$, for any $h$ such that
$L(h(x), y) \leq M$ for all $(x, y)$, we can write
\begin{align*}
& \cL(\sD_T, h)\\
& \leq  \sfd_\alpha(\sD_T \parallel \sD)^{\frac{\alpha -
  1}{\alpha}} \mspace{-8mu} \left[ \sum_{(x, y)} \sD(x,
  y) L(h(x), y)^{\frac{\alpha}{\alpha - 1}} \mspace{-4mu} \right]^{\mspace{-8mu}\frac{\alpha -
  1}{\alpha}}\\
& = \sfd_\alpha(\sD_T \parallel \sD)^{\frac{\alpha - 1}{\alpha}}
  \mspace{-8mu} \left[ \sum_{(x, y)} \sD(x,
  y) L(h(x), y) L(h(x), y)^{\frac{1}{\alpha - 1}} \mspace{-4mu} \right]^{\mspace{-8mu}\frac{\alpha -
  1}{\alpha}}\\
& \leq \sfd_\alpha(\sD_T \parallel \sD)^{\frac{\alpha - 1}{\alpha}} \left[ \sum_{(x, y)} \sD(x,
  y) L(h(x), y) M^{\frac{1}{\alpha - 1}} \right]^{\mspace{-8mu}\frac{\alpha -
  1}{\alpha}}\\
& \leq \Big[ \sfd_\alpha(\sD_T \parallel \sD) \, \cL(\sD, h)
  \Big]^{\mspace{-6mu}\frac{\alpha - 1}{\alpha}} M^{\frac{1}{\alpha }}.
\end{align*}
Now, by Theorem~\ref{th:mixture}, there exists $z \in \Delta$ and
$\eta > 0$ such that $\cL(\sD, h_z^\eta) \leq \e + \delta$ for any
mixture distribution $\sD \in \cD$. Thus, in view of the previous
inequality, we can write,for any $\sD \in \cD$,
\begin{equation*}
\cL(\sD_T, h_z^\eta) \leq \Big[(\e + \delta) \, \sfd_\alpha(\sD_T \parallel \sD)
 \Big]^{\frac{\alpha - 1}{\alpha}} M^{\frac{1}{\alpha }}.
\end{equation*}
Taking the infimum of the right-hand side over all $\sD \in \cD$
completes the proof.
\end{proof}

\begin{reptheorem}{th:estimate}
Let $\sD_T$ be an arbitrary target distribution.
Then, for any $\delta > 0$, there exists $\eta > 0$ and
$z \in \Delta$, such that the following inequality holds for any
$\alpha > 1$:
\begin{equation*}
\cL(\sD_T, \h h_z^\eta) 
\leq \Big[(\h \e + \delta) \, \sfd_\alpha(\sD_T \parallel \h \cD)
\Big]^{\frac{\alpha - 1}{\alpha}} M^{\frac{1}{\alpha }}.
\end{equation*}
\end{reptheorem}
\begin{proof}
By the first part of the proof of Theorem~\ref{th:arbitrary},
for any $k \in [p]$ and $\alpha > 1$, the following
inequality holds:
\begin{align*}
\cL(\h \sD_k, h_k) 
& \leq \Big[\sfd_\alpha(\h \sD_k \parallel  \sD_k) \, \cL(\sD_k, h_k)
\Big]^{\frac{\alpha - 1}{\alpha}} M^{\frac{1}{\alpha }}\\
& \leq \Big[\e \, \sfd_\alpha(\h \sD_k \parallel  \sD_k) 
\Big]^{\frac{\alpha - 1}{\alpha}} M^{\frac{1}{\alpha }} \leq \h \e.
\end{align*}
We can now apply the result of Theorem~\ref{th:arbitrary} (with
$\h \e$ instead of $\e$ and $\h \sD_k$ instead of $\sD_k$).
In view that, there exists $\eta > 0$ and $z \in \Delta$
such that
\begin{equation*}
\cL(\sD_T, h_z^\eta) 
\leq \Big[(\h \e + \delta) \, \sfd_\alpha(\sD_T \parallel \h \sD)
\Big]^{\frac{\alpha - 1}{\alpha}} M^{\frac{1}{\alpha }},
\end{equation*}
for any distribution $\h \sD$ in the family $\h \cD$. Taking the
infimum over all $\h \sD$ in $\h \cD$ completes the proof.
\end{proof}

\newpage
\section{Stationary Point Analysis}
\label{app:stationary}

In this section we explicitly analyze the stationary point conditions
for the min-max optimization problem:
\begin{align*}
  \min_{z} & \ f(z) \\
\text{s.t.} & \ g_k(z)  \leq 0, \forall k \in [p], \ h(z)  = 0,
\end{align*}
where
\begin{align*}
  f(z) &=\max_{k\in[p]}\cL(\sD_k,h_z^\eta)-\cL(\sD_z,h_z^\eta), \\
  g_k(z) &=-z_k , \ 
  h(z) =\sum_{k=1}^p z_k-1.
\end{align*}

For simplicity, we assume $\sD_k(x) >0, \forall k\in [p], \forall x\in \cX$, and let $\eta=0$. 
D.C. algorithm converges to a stationary point $z^* \in \Delta $. 
Under suitable constraint qualification, there exists KKT multipliers 
$\alpha_k,k\in[p]$ and $\beta$ such that 
\begin{align}
  \label{eq:kkt}
 - \left[\sum_{k=1}^p\alpha_k \nabla g_k \left(z^*\right)+
  \beta\nabla_z h\left(z^* \right)\right]	
  &\in \partial_z f(z^*) \\
  \alpha_k g_k(z^*)=0,\ \alpha_{k} &\ge0,\ k\in[p] \nonumber
\end{align}
where the set of sub-gradients of $f(z)$ is subset of a convex hull:
\begin{equation*}
  \partial_{z}f(z) \subseteq \text{Co}
  \left\{ \nabla_{z}\cL(\sD_k,h_z)-\nabla_z \cL(\sD_z,h_z), k\in[p] \right\}.
\end{equation*}
Observer that, for any fixed $x\in\cX$ and 
$z\in\mathbb{R}^p$ , $z^T\nabla_z h_z(x)=0$. 
It follows that $z^T\nabla_{z}\cL(\sD_k,h_z)=0,\ \forall k\in[p]$. 
Thus, by taking inner product of $z^{*}$ and elements in $\partial_{z}f(z^{*})$, 
we obtain
\begin{equation*}
 z^{*T}\partial_zf(z^*)=- z^{*T}\nabla_z\cL(\sD_{z^{*}},h_{z^{*}})
=- \cL(\sD_{z^*},h_{z^*}).
\end{equation*}
On the other hand, 
\begin{equation*}
  z^{*T}\left[\sum_{k=1}^{p}\alpha_{k}\nabla g_{k}(z^{*})+\beta\nabla_{z}h(z^{*})\right]
 =-\sum_{k}\alpha_{k}z_{k}^{*}+\beta=\beta.
\end{equation*}
Therefore $\beta=\cL(\sD_{z^{*}},h_{z^{*}})$, thus we obtain the following
simplified stationary point conditions: there exits $\alpha_k,k\in[p]$ such that
\begin{align*}
  \left[\cL(\sD_{z^{*}},h_{z^{*}})-\alpha_{1},\ldots,\cL(\sD_{z^{*}},h_{z^{*}})-\alpha_{p}\right]^{T}	
  &\in-\partial_{z}f\left(z^{*}\right) \\ 
  \alpha_{k}z_{k}^{*}=0,\ \alpha_{k}	&\ge0,\ k\in[p].
\end{align*}

\newpage
\section{More Experiment Results}
\label{app:more_exp}
In this section we provide more experiment results to show that our 
distribution weighted predictor \texttt{DW} is robust across various test data mixtures.
\begin{table}[t]
\begin{center}
\scriptsize{
\begin{tabular}{l cccc}
\toprule
& \multicolumn{4}{c}{Test Data}\\
\cline{2-5}
& \texttt{KB}& \texttt{KE}& \texttt{DB}& \texttt{DE} \\
\midrule
\texttt{K}      &       {1.87$\pm$0.08} &       {1.57$\pm$0.06} &       {2.25$\pm$0.08} &       {1.94$\pm$0.10}\\
\texttt{D}      &       {2.12$\pm$0.07} &       {2.11$\pm$0.05} &       {1.95$\pm$0.06} &       {1.94$\pm$0.06}\\
\texttt{B}      &       {1.96$\pm$0.07} &       {2.21$\pm$0.06} &       {1.87$\pm$0.07} &       {2.13$\pm$0.05}\\
\texttt{E}      &       {2.05$\pm$0.05} &       {1.60$\pm$0.05} &       {2.36$\pm$0.07} &       {1.91$\pm$0.07}\\
\texttt{unif}   &       {1.74$\pm$0.05} &       {1.62$\pm$0.04} &       {1.85$\pm$0.05} &       {1.73$\pm$0.06}\\
\texttt{KMM}    &       1.78$\pm$0.12   &       1.65$\pm$0.10   &       1.97$\pm$0.13   &       1.88$\pm$0.08\\
\texttt{DW}     &       {\bf1.59$\pm$0.05}      &       {\bf1.47$\pm$0.04}      &       {\bf1.75$\pm$0.05}      &       {\bf1.64$\pm$0.05}\\
\bottomrule
\end{tabular}
}
\end{center}
\caption{MSE on the sentiment analysis dataset. Using the learned $z$, our distribution weighted method \texttt{DW} outperforms all competing baselines in terms of robustness across test domains. We compare against source only baselines for each domain, \texttt{K, D, B, E},  a uniform weighted predictor \texttt{unif}, and \texttt{KMM}. }
\label{table:sa_more_pairs}	
\end{table}

\begin{table*}[t]
\begin{center}
\scriptsize{
\begin{tabular}{l cccccccccc}
\toprule
& \multicolumn{9}{c}{Test Data}\\
\cline{2-11}
& \texttt{\textbf{K}DBE}& \texttt{K\textbf{D}BE}& \texttt{KD\textbf{B}E}& \texttt{KDB\textbf{E}} & \texttt{\textbf{KD}BE}& \texttt{\textbf{K}D\textbf{B}E}& \texttt{\textbf{K}DB\textbf{E}}& \texttt{K\textbf{DB}E} & \texttt{K\textbf{D}B\textbf{E}} & \texttt{KD\textbf{BE}}\\
\midrule
\texttt{K}      &       {1.78$\pm$0.05} &       {1.94$\pm$0.10} &       {1.96$\pm$0.08} &       {1.84$\pm$0.07} &       {1.86$\pm$0.10} &       {1.87$\pm$0.07} &       {1.79$\pm$0.08}      &       {1.96$\pm$0.10} &       {1.89$\pm$0.10} &       {1.89$\pm$0.08}\\
\texttt{D}      &       {2.02$\pm$0.10} &       {1.98$\pm$0.10} &       {2.06$\pm$0.11} &       {2.05$\pm$0.09} &       {2.01$\pm$0.13} &       {2.05$\pm$0.12} &       {2.04$\pm$0.12}      &       {2.03$\pm$0.12} &       {2.02$\pm$0.13} &       {2.06$\pm$0.12}\\
\texttt{B}      &       {2.01$\pm$0.12} &       {2.01$\pm$0.14} &       {1.94$\pm$0.14} &       {2.06$\pm$0.11} &       {2.01$\pm$0.15} &       {1.98$\pm$0.14} &       {2.05$\pm$0.13}      &       {1.98$\pm$0.15} &       {2.04$\pm$0.14} &       {2.01$\pm$0.13}\\
\texttt{E}      &       {1.93$\pm$0.08} &       {2.04$\pm$0.10} &       {2.08$\pm$0.10} &       {1.89$\pm$0.08} &       {2.00$\pm$0.10} &       {2.01$\pm$0.09} &       {1.91$\pm$0.08}      &       {2.08$\pm$0.10} &       {1.97$\pm$0.08} &       {1.99$\pm$0.08}\\
\texttt{unif}   &       {1.69$\pm$0.06} &       {1.74$\pm$0.07} &       {1.75$\pm$0.08} &       {1.70$\pm$0.06} &       {1.72$\pm$0.09} &       {1.72$\pm$0.08} &       {1.69$\pm$0.07}      &       {1.75$\pm$0.08} &       {1.72$\pm$0.08} &       {1.73$\pm$0.08}\\
\texttt{KMM}    &       1.83$\pm$0.12   &       1.92$\pm$0.14   &       1.87$\pm$0.15   &       1.85$\pm$0.13	&       1.85$\pm$0.16   &       1.86$\pm$0.14   &       1.85$\pm$0.15   &       1.90$\pm$0.14   &       1.89$\pm$0.16   &       1.90$\pm$0.14\\
\texttt{DW}     &       {\bf1.55$\pm$0.08}      &       {\bf1.62$\pm$0.08}      &       {\bf1.59$\pm$0.09}      &       {\bf1.56$\pm$0.08}      &       {\bf1.58$\pm$0.10}      &   {\bf1.57$\pm$0.10}       &       {\bf1.55$\pm$0.09}      &       {\bf1.61$\pm$0.10}      &       {\bf1.59$\pm$0.08}      &       {\bf1.58$\pm$0.09}\\
\bottomrule
\end{tabular}
}
\end{center}
\caption{MSE on the sentiment analysis dataset. Using the learned $z$, our distribution weighted method \texttt{DW} outperforms all competing baselines in terms of robustness across test domains. We compare against source only baselines for each domain, \texttt{K, D, B, E},  a uniform weighted predictor \texttt{unif}, and \texttt{KMM}. }
\label{table:sa_more_lambdas}	
\end{table*}

Table~\ref{table:sa_more_pairs} shows MSE on various combinations of 
two domains that is not included in Table~\ref{table:sa}. 

Table~\ref{table:sa_more_lambdas} reports MSE on additionally test domain mixtures.
The first four target mixtures correspond to various orderings of $\{0.4, 0.2, 0.2, 0.2\}$.
The next six target mixtures correspond to various orderings of $\{0.3, 0.3, 0.2, 0.2\}$. 
In column title we bold the domain(s) with highest weight.
\end{document}